%% file: 2015_path.tex
\documentclass[11pt,twoside]{article} % For LaTeX2e
\usepackage{balance}
\usepackage{jmlr2e}
\usepackage{multirow}
\usepackage{resizegather}
\usepackage{flushend}
\usepackage{graphicx}
\usepackage[textwidth=2cm,colorinlistoftodos]{todonotes}
\usepackage{tikz}
\usetikzlibrary{arrows,shadows,shapes,backgrounds,decorations,snakes,fit}
\usetikzlibrary{calc,decorations.pathmorphing,patterns}
\usetikzlibrary{intersections}
\usetikzlibrary{shadows,matrix}

\input{preamble.tex}

\ShortHeadings{On a Family of Decomposable Kernels on Sequences}{Baisero, Pokorny, Ek}
\firstpageno{1}

\begin{document}

\title{On a Family of Decomposable Kernels on Sequences}

\author{\name Andrea Baisero \email andrea.baisero@ipvs.uni-stuttgart.de\\
  \addr Institute for Parallel and Distributed Systems\\
  Machine Learning and Robotics, University of Stuttgart\\
  Stuttgart, Germany
  \AND
  \name Florian T. Pokorny \email fpokorny@csc.kth.se\\
  \name Carl Henrik Ek \email chek@csc.kth.se \\
  \addr School of Computer Science and Communication\\
       Royal Institute of Technology, KTH\\
       Stockholm, Sweden
}

\maketitle

\begin{abstract}
In many applications data is naturally presented in terms of orderings of some basic elements or symbols. Reasoning about such data requires a notion of similarity capable of handling sequences of different lengths. In this paper we describe a family of Mercer kernel functions for such sequentially structured data. The family is characterized by a decomposable structure in terms of symbol-level and structure-level similarities, representing a specific combination of kernels which allows for efficient computation. We provide an experimental evaluation on sequential classification tasks comparing kernels from our family of kernels to a state of the art sequence kernel called the Global Alignment kernel which has been shown to outperform Dynamic Time Warping.
\end{abstract}


\section{Introduction}\label{sec:introdution}
Many types of data have inherent sequential structure. Sequences of letters in computational linguistics, series of images in computer vision or cell structures in computational biology and arbitrary data sets depending on a parameter such as time provide familiar examples of such data. It is hence not surprising that there exists a significant amount of work focused on representing such data. In \citep{Rieck:2011tz} the author reviews and broadly categorizes sequential similarity measures into three main categories: \emph{bag-of-words}, \emph{edit-distance} and \emph{string-kernel} based methods. Bag-of-words \citep{Harris:1970wr} based similarity measures translate the notion of a sequence to a distribution over certain sub-sequences (\emph{i.e.} words in natural language processing) of the sequence itself, meaning that such measures only encode the sequential structure up to the length of the sub-sequence and disregard information about word order. As such, Bag-of-words methods require us to be able to identify significant sub-sequences (the words), which is not always obvious for sequences arising outside natural language. Nevertheless, this approach captures some structure and, as the sequential data is translated into a vector space whose basis consists of elementary subsequences, it allows us to interpret the data and enables us to use well-developed learning methods for such vectorial data. Techniques based on edit distances \citep{Damerau:1964hu,Levenshtein:1966ts} relate sequences by defining a transformation from one sequence to the other and associating a cost to the transformation. Edit distances can be very useful if the notion of cost with respect to different transformations is well grounded. The third category refers to (dis)similarity measures defined by implicitly specifying an inner-product space through a kernel function between sequences. String kernels \citep{Lodhi:2002ts,Rousu:2006vw} were proposed in Computational Linguistics, where data consists of sequences (text) of discrete symbols (letters). The (dis)similarity measure is defined in terms of ``gaps'' between symbols in the two sequences. String kernels are a specific instance of a larger class of kernel functions referred to as rational kernels \citep{Cortes:2004vf}. Rational kernels are related to weighted automata \citep{Mohri:2009ed} and define inner products from the specific sequential structure described by the automata. In this paper, we will focus on a new family of kernel based (dis)similarity measures.

The contributions of this work are in particular:
\emph{a)} A generic approach for the construction of sequence kernels which scales $\mathcal{O}(nm)$ in the lengths $n, m$ of the input strings. \emph{b)} The kernel decomposes intuitively into structure-level and symbol-level similarities. Compared to previous approaches, the structure of the symbol space can be encoded by any Mercer kernel.
\emph{c)} We show that a recently proposed intuitive (dis)similarity measure on sequences \citep{Baisero:2013tm}, is positive definite kernel and falls into our class.
\emph{d)} We compare and evaluate several kernels from our family which perform favourably against the state of the art Global Alignment kernel.

\section{Related work}
The three main sequence similarity approaches discussed above are all based on the concept that sequence similarity is defined in terms of discrete unordered symbols, and the similarity between two symbols $a, b, \in \Sigma$ is typically is defined by zero if $a\neq b$ and one otherwise. However, for many types of data the symbol space $\Sigma$ might be continuous, and we might in fact have a natural similarity measure on $\Sigma$ itself. As an example, consider the problem of matching two discretized waveforms $\alpha=[\alpha_1, \ldots, \alpha_n]$, $\beta=[\beta_1, \ldots, \beta_m]$ where $\alpha_i, \beta_i\in \bR=\Sigma$ and where there exists a natural distance $\norm{a-b}$ for $a, b\in \Sigma=\bR$. A popular similarity measure closely related to edit distances is Dynamic Time Warping \citep{Sakoe:1978jp,Muller:2007us}. It provides a similarity measure based on the cost of aligning two sequences such that the sum of matching each element is minimized. This measure does not by itself correspond to a positive definite kernel function \citep{Bahlmann:2002hi} and hence lacks a geometrical interpretation. One approach has been to use the dynamic time warping distance inside a radial basis exponential kernel function \citep{Lei:2007kb,Bahlmann:2002hi}. However this still suffers from the drawback that dynamic time warping is not a kernel itself. Even though non-positive kernels have been shown to be useful \citep{Haasdonk:2005di} in practice, they lack a geometrical interpretation and the mathematical justification which makes the use of kernel methods so appealing.

Motivated by the intuition for the definition of dynamic time warping, \citep{Cuturi:2007db} developed a related similarity measure which in fact corresponds to a valid kernel function for sequences. Here a (dis)similarity function is defined by summarizing all possible alignments between two sequences through a `soft-min' rather than using only the minimal cost alignment as in dynamic time warping. Importantly, compared to previous kernels on sequences, this kernel is capable of incorporating a structured non-discrete symbol space $\Sigma$.  The resulting kernel is referred to as the Global Alignment kernel and was shown to outperform Dynamic Time Warping for sequence classification. However, to be a valid Mercer kernel, the structure of the symbol space $\Sigma$ have to be induced by a specific class of kernel functions. Further, it strongly favors small sequence perturbations over larger perturbations which reduces the ability of the kernel to generalize example data. Some of these issues have been addressed in \citep{Cuturi:2010vq} where only a subset of the possible alignments contributes to the inner-product.

Another approach was taken in \citep{Baisero:2013tm}, where we proposed a (dis)similarity measure called the Path kernel. Just like the Global Alignment kernel, this kernel is defined by reasoning about the (dis)similarity of all possible alignments of two sequences. In experiments, the Path kernel performed better than the Global Alignment kernel for a set of experiments both with respect to accuracy and computational cost. We will show that this kernel naturally falls into a class of kernels that we will define in this work, thus proving that it is positive semi-definite\footnote{Note that, the Path kernel defined in \citep{Baisero:2013tm} is not related to the special Rational kernel proposed in \citep{Takimoto:2003uq}, which is also referred to as a Path kernel}.

\section{On the construction of sequence kernels}\label{sec:sequencekernels}
We are interested in finite sequences $\mathbf{s} = (\mathbf{s}_1,\ldots,\mathbf{s}_{|\mathbf{s}|})$, with symbols $\mathbf{s}_i\in \Sigma$, belonging to a symbol space $\Sigma$ which can be discrete or continuous. We denote the set of such finite sequences by $Seq(\Sigma)$ and are interested in studying combinations of Mercer kernel functions on symbols $k_{\Sigma}: \Sigma \times \Sigma \rightarrow \mathbb{R}$ that yield valid Mercer kernels on a sequential level $Seq(\Sigma)$. We follow the convention of calling a kernel $k:X\times X \to \bR$ positive definite if $\sum_{i, j=1}^n c_i k(x_i, x_j) c_j \ge 0$ for any finite subset $\{x_1, \ldots, x_n\}\subset X$, $n\in\mathbb{N}$ and any $\{c_1, \ldots, c_n\}\subset \bR$. Let us now describe a novel general approach towards the construction of such kernels for sequences belonging to $Seq(\Sigma)$:
\begin{lem}
Let $k_{\Sigma}:\Sigma\times \Sigma\to \bR$ be a continuous positive definite kernel on $\Sigma$, where $\Sigma$ is a separable metric space and let $k_{S}:\bN \times \bN\to \bR$ be a positive definite kernel on integers. Then the kernel
\begin{equation}
  k(\mathbf{s}, \mathbf{t}) = \sum_{i=1}^{|\mathbf{s}|}\sum_{j=1}^{|\mathbf{t}|} k_{\Sigma}(\mathbf{s}_i, \mathbf{t}_j)k_{S}(i, j),
\label{eq:sequencekernel}
\end{equation}
defined for any finite sequences $\mathbf{s}, \mathbf{t}\in Seq(\Sigma)$, $\mathbf{s}=(\mathbf{s}_1, \ldots, \mathbf{s}_{|\mathbf{s}|})$ and $\mathbf{t}=(\mathbf{t}_1, \ldots, \mathbf{t}_{|\mathbf{t}|})$ is also positive definite.
\label{lem:lemma1}
\end{lem}
\begin{proof}
Observe that both $k_{\Sigma}$ and $k_{S}$ can be trivially extended to kernels on $\Sigma\times \bN$ by $K_1((\mathbf{s}, i), (\mathbf{t}, j))=k_{\Sigma}(\mathbf{s}_i, \mathbf{t}_j)$, $K_2((\mathbf{s}, i), (\mathbf{t}, j))=k_{S}(i, j)$ for $\mathbf{s}_i, \mathbf{t}_j\in \Sigma$ and $i, j\in \bN$. Now $K((\mathbf{s}, i), (\mathbf{t}, j))=K_1((\mathbf{s}, i), (\mathbf{t}, j)) K_2((\mathbf{s}, i), (\mathbf{t}, j))$ is a positive kernel on $U=\Sigma\times \bN$. Let $X, Y$ be finite subsets of $U$. According to Lemma 1, \citep{Haussler1999a}, the kernel
\[
L(X, Y) = \sum_{x\in X,  y\in Y} K(x, y)
\]
is then also positive definite. Note that a sequence $\mathbf{s}=(\mathbf{s}_1, \mathbf{s}_2, \ldots, \mathbf{s}_n)$ corresponds to a subset $X=\{(\mathbf{s}_1, 1), (\mathbf{s}_2, 2), \ldots, (\mathbf{s}_n, n)\}$, and thus the above kernel is positive definite.
\end{proof}
Note that any countable discrete space and any finite dimensional vector space can be given the structure of a separable metric space.
If $\Sigma$ is discrete and countable, any kernel on $\Sigma$ is trivially continuous with respect to the discrete topology.
Note also that, while the above result readily follows from the work on convolution kernels by \citep{Haussler1999a}, the above natural class of kernels has -- to the best of our knowledge -- not been
studied or formulated in this manner. This might be partially, because classical kernels coming from natural language processing often only consider similarity measures on the symbol space $\Sigma$ directly.

We observe that the family of kernels described above relates all pairs of the input sequences' symbols using $k_{\Sigma}$ and adjusts these values according to similarity of positions of the symbols within the sequences, as measured by $k_S$.
An added benefit of the proposed family of kernels is their relative computational simplicity, since kernel evaluations scale like $\mathcal{O}(|s||t|)$
in the length of the input strings $\mathbf{s}, \mathbf{t}$. Noting that,
\begin{equation}
    k(\mathbf{s},\mathbf{t}) = tr(K_{\Sigma}(\mathbf{s}, \mathbf{t})^TK_{S}(|\mathbf{s}|,|\mathbf{t}|)),
%  \mathbf{1}^{\textrm{t}}\left(K_\Sigma(\mathbf{s},\mathbf{t})\circ K_\Gamma(|\mathbf{s}|,|\mathbf{t}|) \right)\mathbf{1},
\end{equation}
where $[K_\Sigma(\mathbf{s},\mathbf{t})]_{ij} = k_\Sigma(\mathbf{s}_i,\mathbf{t}_j)$, $[K_S(|\mathbf{s}|,|\mathbf{t}|)]_{ij} = k_S(i,j)$, $i=1, \ldots, |\mathbf{s}|$, $j=1, \ldots, |\mathbf{t}|$ and $tr$ denotes the trace, we observe that the matrix $K_{S}$ can be pre-computed once the maximal length of any sequence in a data-set is known. The evaluation of the kernel is then just a trace of a matrix product which can be efficiently implemented.

In a typical scenario $k_{S}$ and $k_{\Sigma}$ might also depend on parameters $\theta_S\in \bR^n$ and $\theta_{\Sigma}\in \bR^m$. These parameters can be set through cross-validation but
they can also be learned if the gradients of the kernel functions with respect to these parameters can be computed. If we wish to use the kernel to represent a functional relationship $f:\mathbf{s}\mapsto\mathbf{y}$, where $\mathbf{y}\in\mathbb{R}^d$, we can encode a preference over the mapping $f$ by a Gaussian process \citep{Rasmussen:2005te}. If the co-variance in the output space is encoded by a sequence kernel $k$, the parameters, $\theta_{\Sigma}$ and $\theta_S$ can then be learned by maximizing the marginal likelihood of the model. In order to accommodate classification in a Gaussian process framework, the regression noise is usually squashed
\citep{Rasmussen:2005te} rendering the integration required to reach the marginal likelihood infeasible. However, it has been observed that learning the parameters for a classification task with $1-C$ encoding, where each class is encoded using a binary variable, and a Gaussian noise assumption works well in practice \citep{Kapoor:2009il}.

%\todo[inline]{Andrea: why is the folling already using the path kernel's notation? Wouldn't it be better if we still used the generic / neutral notation above? Moreover, describing the computational efficiency using the trace formulation is imho counterproductive: that's not really how it is computed, and only taking the trace of the multiplied matrix is definitely not the best way, because that implies that all the elements on the off-diagonal would be computed for nothing. I suggest we focus on the fact that $k_S$ can be precomputed adequately, and then used according to the original description of the family of kernels, i.e. according to the Hadamart product.}

\section{Examples of Sequence Kernels}\label{sec:kernels}
\begin{figure*}
  \begin{center}
    \includegraphics[width=0.19\textwidth]{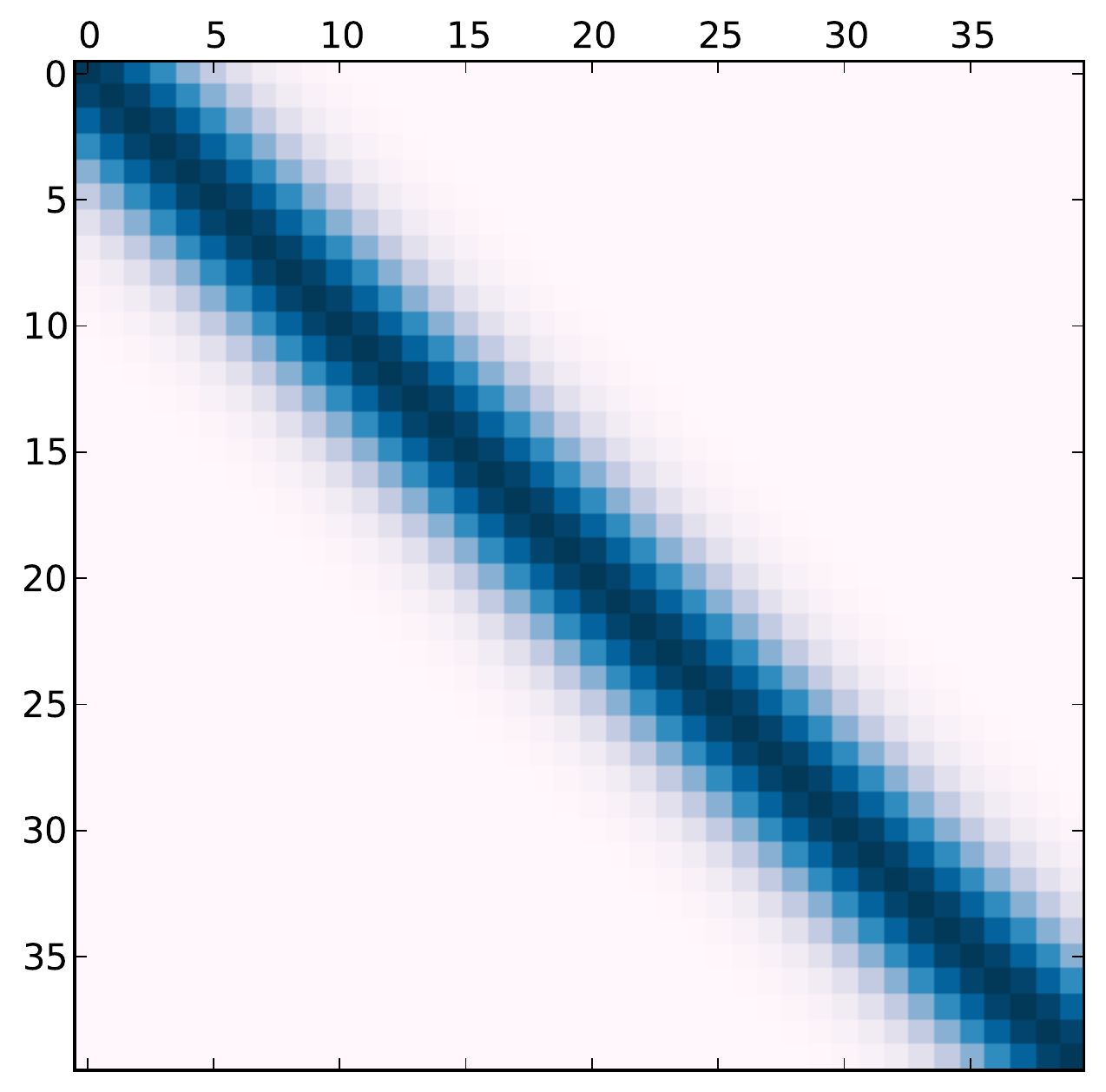}
    \includegraphics[width=0.19\textwidth]{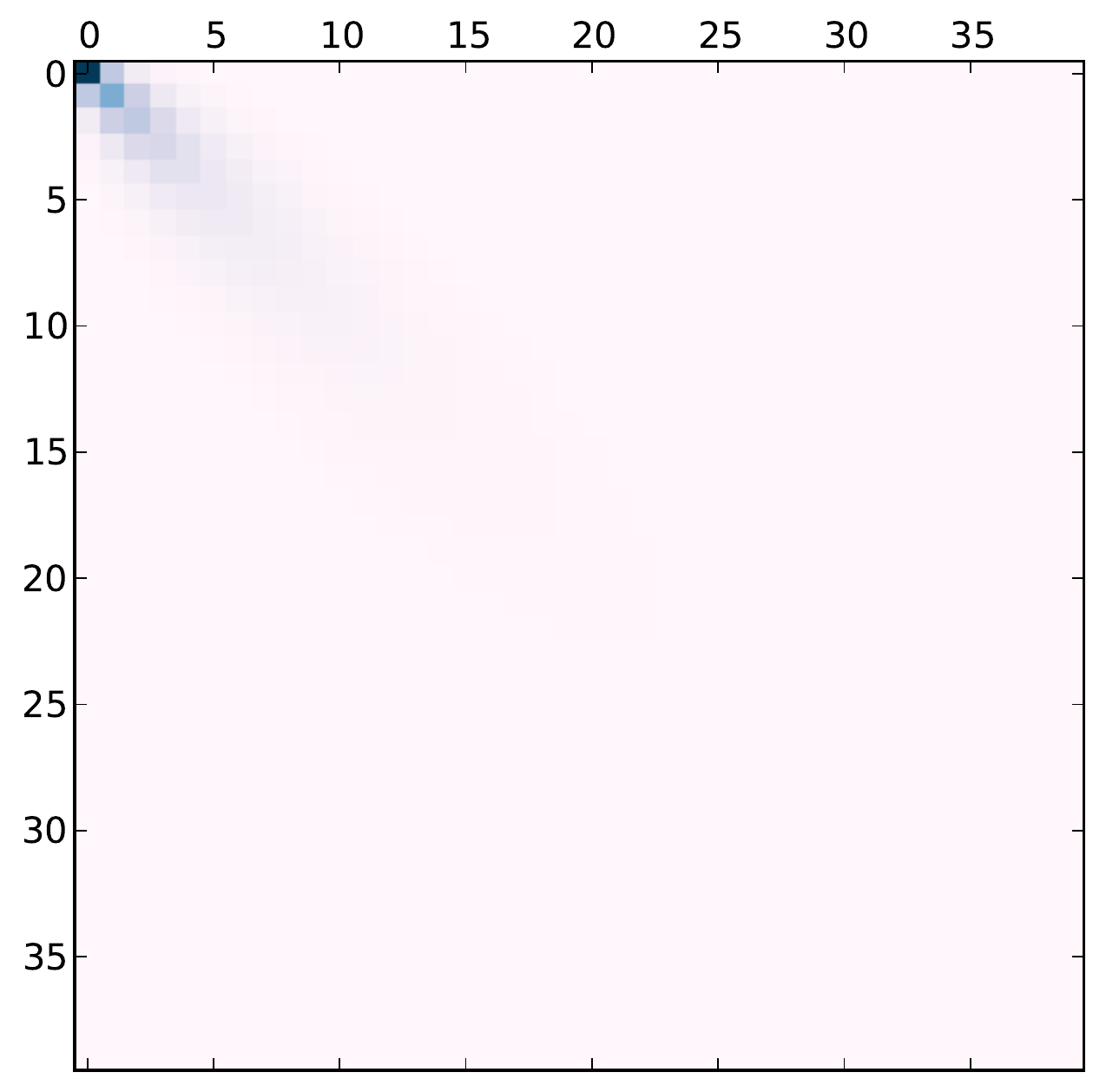}
    \includegraphics[width=0.19\textwidth]{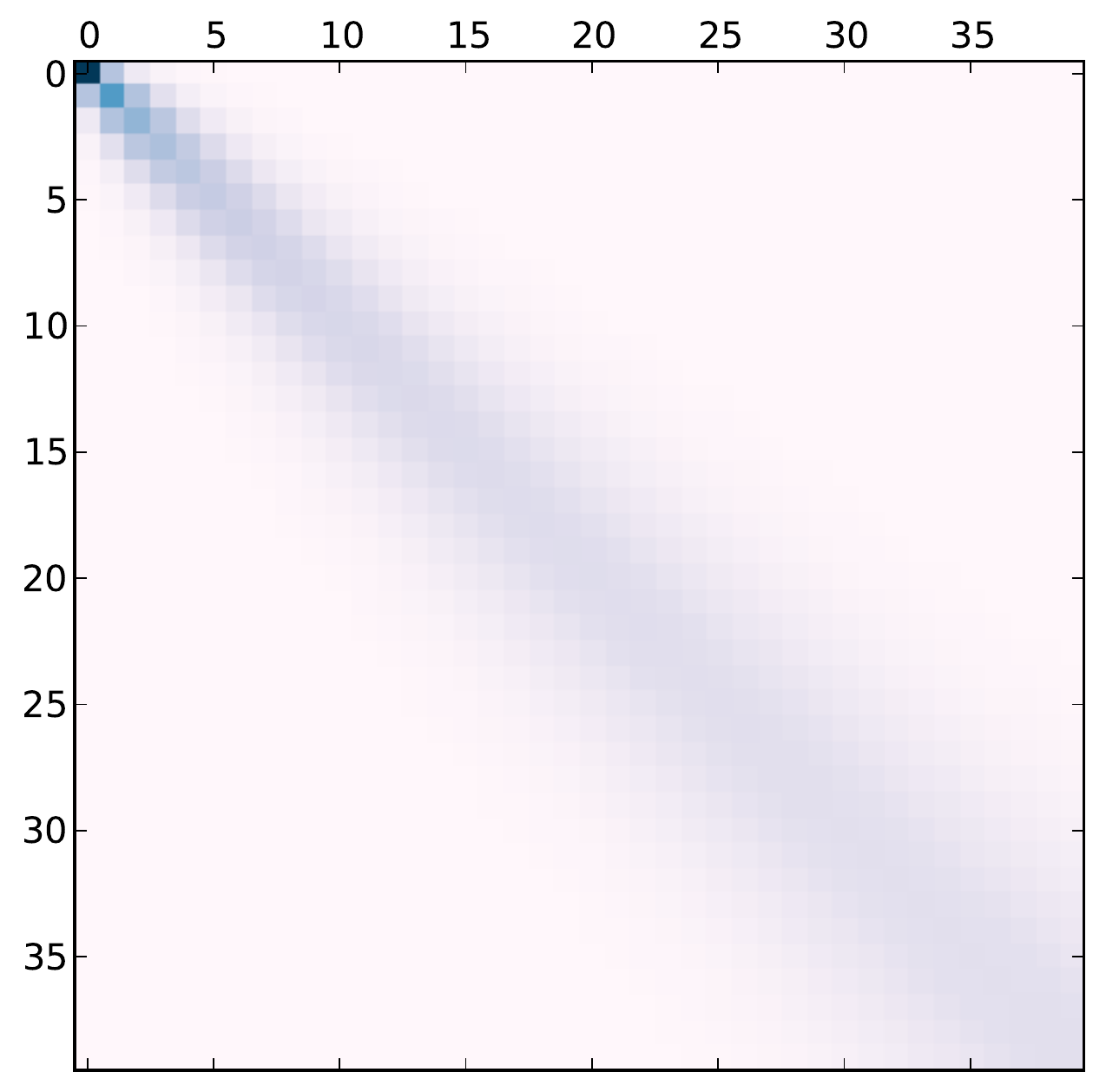}
    \includegraphics[width=0.19\textwidth]{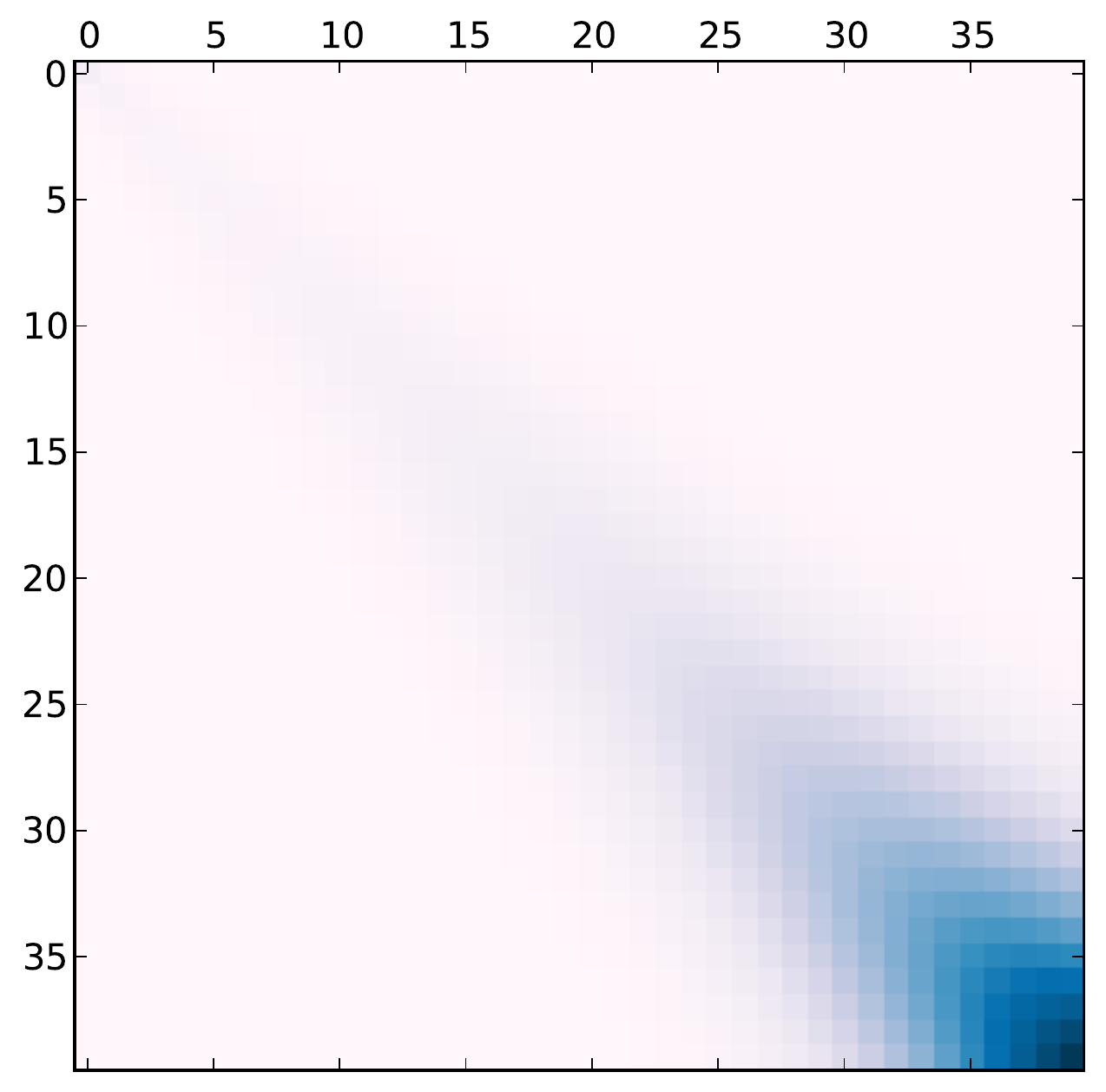}
    \includegraphics[width=0.19\textwidth]{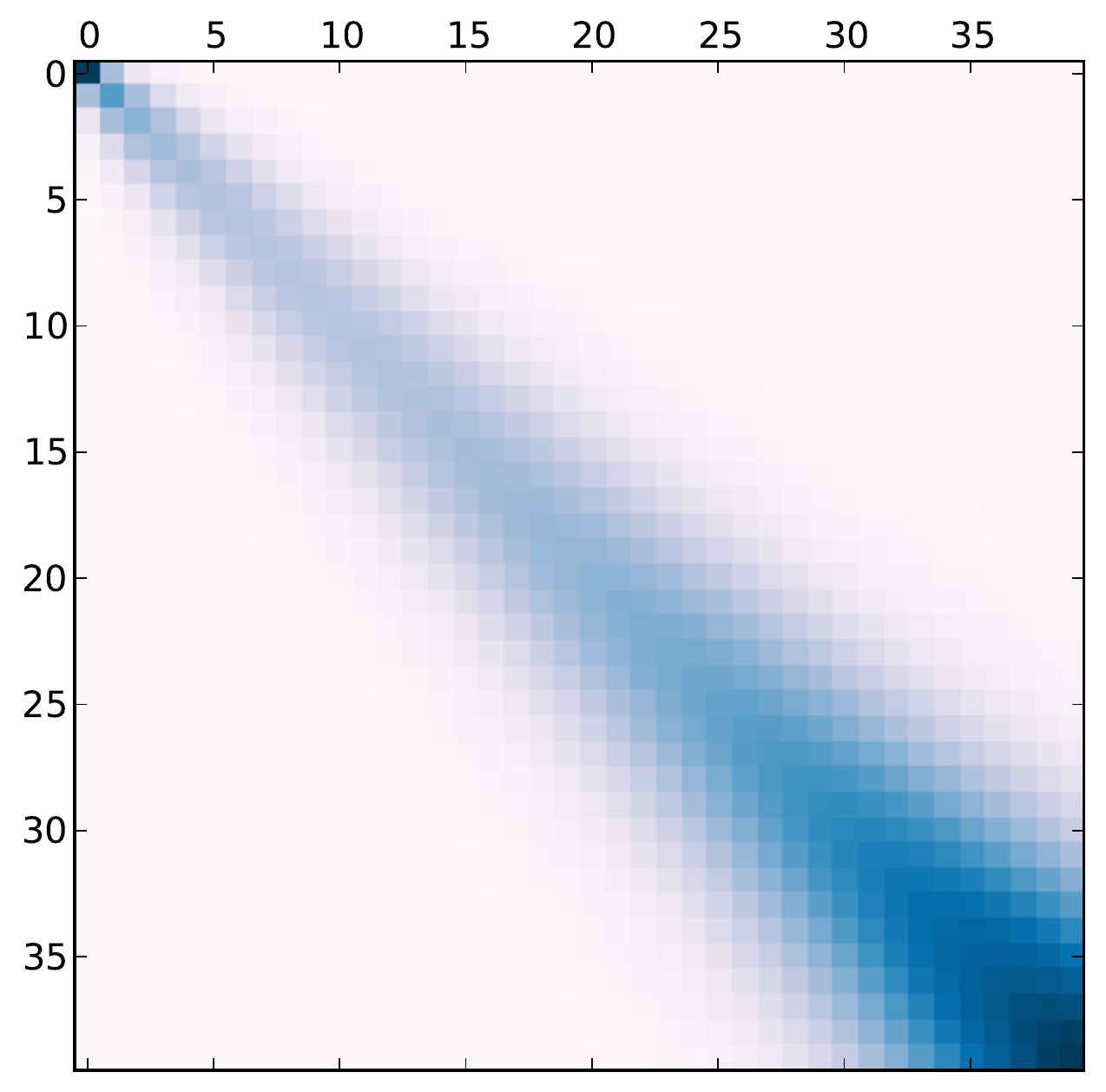}
  \end{center}
  \caption{{\it
      The above figure shows five different structure kernel matrices. The left-most image depicts the exponential kernel while the remaining four show the structure kernel $k_{\Gamma}$ for the path sequence kernel for varying parameter values. The values of $C_{\text{d}} = \{0.3,0.35,0.35,0.3 \}$ and $C_{\text{hv}} = \{0.3,0.33,0.37,0.37\}$. Changing the parameters $C_{\text{d}}$ and $C_{\text{hv}}$ emphasizes different parts of the alignment highlighting the non-stationary structure of how the contribution of different parts of a sequences is accumulated.}}
  \label{fig:kernels}
\end{figure*}

In this paper we will focus on kernels
from the family in Lemma \ref{lem:lemma1}. A straightforward approach to formulate such a sequence kernel would be to pick a familiar kernel $k_S$, where $\abs{i-j}$ determines the impact on $k$. This can be implemented by a stationary kernel $k_{S}$ such as an exponential kernel:
\begin{cor}
For $\alpha>0$, the function $k_e:Seq(\Sigma)\times Seq(\Sigma) \to \bR$ given by,
\begin{equation}
    k_e(\mathbf{s}, \mathbf{t}) = \sum_{i=1}^{|\mathbf{s}|}\sum_{j=1}^{|\mathbf{t}|} k_{\Sigma}(\mathbf{s}_i, \mathbf{t}_j) e^{-\frac{\norm{i-j}^2}{\alpha}},
\end{equation}
is a valid kernel on $Seq(\Sigma)$ corresponding to the exponential structure kernel on $\bN$.
\end{cor}
Similarly, we can now define kernels by defining a kernel $k_S$ on integers by restricting any known kernel on $\bR$ to the integers. Examples include the
polynomial kernels $k_S(i, j) = (ij+c)^d$, the perceptron kernel, \emph{etc}.
While the above kernel follows readily from the definition of our class of sequence kernels, we would now like to focus on a secondary viewpoint which, as we will show, also leads to a sequence kernel as in Lemma \ref{lem:lemma1}. In \citep{Baisero:2013tm}, the authors proposed a novel (dis)similarity measure $k_p:Seq(\Sigma)\times Seq(\Sigma)\to \bR$ which can be defined most elegantly in a recursive fashion as,
\begin{align}
  k_p(\mathbf{s},\mathbf{t}) &= \begin{cases} \begin{aligned}
      &k_{\Sigma}(\mathbf{s}_1,\mathbf{t}_1) \\
      &\quad + C_{\text{hv}} k_p(\mathbf{s}_{2:},\mathbf{t}) \\
      &\quad + C_{\text{hv}} k_p(\mathbf{s},\mathbf{t}_{2:}) \\
      &\quad + C_{\text{d}} k_p(\mathbf{s}_{2:},\mathbf{t}_{2:})
    \end{aligned} &
    \begin{aligned}
      &|\mathbf{s}|\ge1 |\mathbf{t}|\ge1\\
      & C_{\text{d}}\ge0\\
      & C_{\text{hv}}\ge0
    \end{aligned}\\
    0 & \text{otherwise,}
  \end{cases}.
  \label{eq:pathkernelrecursive}
\end{align}
where $\mathbf{t_{2:}}$ denotes the sequence obtained by removing the first symbol from $\mathbf{t}\in Seq(\Sigma)$.
The recursive formulation above can be interpreted as accumulating information from all possible alignments of two strings. An alignment of two strings $\mathbf{s}$ and $\mathbf{t}$ is defined by a path $\gamma$ through a matrix $\mathbf{M}$ of size $|\mathbf{s}|\times|\mathbf{t}|$ from element $[\mathbf{M}]_{11}$ to $[\mathbf{M}]_{|\mathbf{s}||\mathbf{t}|}$. Each path defines a different alignment in terms of ``stretches'' of a sequence see Figure~\ref{fig:pathkernel}. Each path is decomposed into series of simple operations which have a different effect, parametrized by $C_{\text{d}}$ and $C_{\text{hv}}$, on the final similarity measure. The cardinality of the set of paths, and therefore of the alignments, for two sequences $\mathbf{s}$ and $\mathbf{t}$ is the Delannoy number $D(|\mathbf{s}|,|\mathbf{t}|)$. In addition to the recursive formulation above, \citep{Baisero:2013tm} also showed that the resulting function can be written in a similar form to Eq.~\ref{eq:sequencekernel} where $k_S=k_{\Gamma}$ and,
\begin{gather}
    k_{\Gamma}(i,j) = \sum_{d=0}^{\min(i,j)-1} C_{\text{hv}}^{i+j-2-2d} C_{\text{d}}^d \frac{(i+j-2-d)!}{(i-1-d)!(j-1-d)!d!}.
\label{eq:gammakernel}
\end{gather}
While the recursive definition of $k_p$ is natural, since it assigns a cost for diagonal and off-diagonal moves in a matrix, the above formula seems rather unintuitive. While \citep{Baisero:2013tm} did not provide a proof of positive definiteness, we will now show that the path kernel $k_p$ does in fact define a positive definite kernel and that $k_p$ naturally falls into the class of kernels considered here.

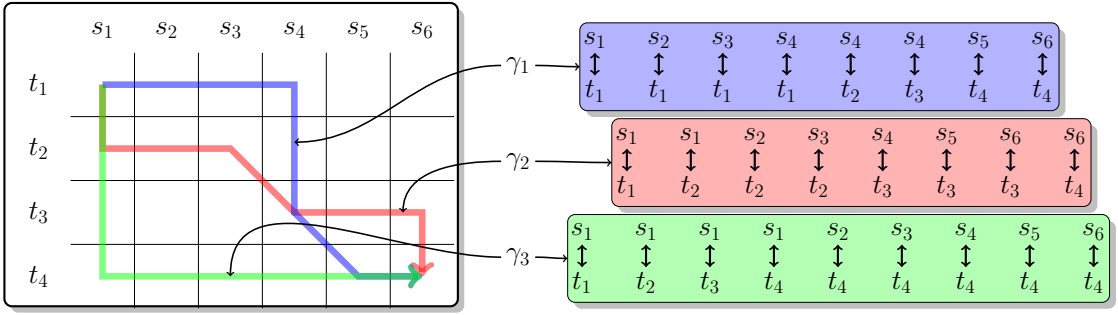
\begin{figure*}
  \begin{center}
    \scalebox{0.50}{\input{includes/pathkernel.tex}}
  \end{center}
  \caption{{\it
    The image above depicts the Path kernel interpreted as an alignments of paths. Two sequences $\mathbf{t}=[\mathbf{t}_1,\ldots,\mathbf{t}_4]$ and $\mathbf{s}=[\mathbf{s}_1,\ldots,\mathbf{s}_6]$ and three separate alignments  $\gamma_{1,2,3}$ are shown. The matrix to the left shows how a subset of the possible alignments between $\mathbf{s}$ and $\mathbf{t}$ can be constructed as \emph{paths} between the top-left to the bottom-right element of the matrix to the right. On the right, the three different alignments generated from the paths on the left are shown.
}}
\label{fig:pathkernel}
\end{figure*}

In order to prove that $k_p$ is a positive definite kernel, we now need to show that $k_{\Gamma}:\mathbb{N}\times\mathbb{N}\rightarrow\mathbb{R}$ is indeed a kernel on integers. First lets recall that the Gamma function $\Gamma:\bC\to \bR$ is defined by,
\begin{equation}
    \Gamma(z) = \int_0^{\infty} t^{z-1} e^{-t} \rd t,
\end{equation}
for $z \in \bC$, $\mathcal{R}e(z)>0$ and that we have $\Gamma(n) = (n-1)!$ for $n\in \bN$. Let us now think of the factorial as a curious example of a positive Mercer kernel on integers:
\begin{lem}
    Let $d\in \mathbb{Z}$ and $X_d = \{ x\in \bN: x\ge \frac{d}{2}\}$.
    The function $k:X_d\times X_d\to \bR$ defined by,
    \[
        k(x, x') = (x+x'-d)!,
    \]
    is a positive definite kernel on $X_d$ corresponding to the feature mapping $\psi_d: X_d \to L^1(\bR_{\ge 0})$ mapping
    $x\in X_d$ to the function $f_x(t) = t^{x-\frac{d}{2}}e^{-\frac{t}{2}}$. I.e., considering
    the standard inner product on $L^1(\bR_{\ge 0})$ given by $\langle f, g \rangle = \int_{0}^{\infty} f(t) g(t) \rd t$
    for two integrable functions $f, g \in L^1(\bR_{\ge 0})$, we have $k(x, x') = \langle f_x, f_y \rangle$.
\end{lem}
\begin{proof}
    The result follows directly from the above integral formula for $\Gamma(x+x'-d+1)$.
\end{proof}
Note that, if $C_{\text{d}}\ge0$, we can use the idea of the above lemma and write $k_{\Gamma}(i, j) = \sum_{d=0}^{min(i, j)-1} \langle \phi_d(i), \phi_d(j)\rangle$, where $\phi_d:X_{2+2d}\to L^1(\bR_{\ge 0})$ is the feature map mapping integers to functions which is given by,
\[
    (\phi_d(i))(t) = \frac{C_{\text{d}}^{\frac{d}{2}}C_{\text{hv}}^{i-1-d}}{(i-1-d)!}t^{i-1-\frac{d}{2}}e^{-\frac{t}{2}},
\]
and $\langle f, g \rangle$ is again the inner-product of functions obtained by integration. $k_d(i, j)=\langle \phi_d(i), \phi_d(j)\rangle$ is a kernel on $X_{2d+2}= \{x \in \bN: x > d\}$. Note that the condition $d\le \min(i, j)-1$, $i, j\in \bN$ in the summation appearing in the definition of $k_{\Gamma}$ is equivalent to $d<i$ and $d<j$, i.e. $i,j\in X_{2d+2}$. Let us call the extension of $k_d$ to $\bN\times \bN$ $\hat{k}_d$, so that $\hat{k}_d(i, j)=0$ if $i$ or $j\notin X_{2d+2}$ and $\hat{k}_d(i, j)=k_d(i, j)$ if $i, j\in X_{2d+2}$. Then $\hat{k}_d:\bN\times \bN \to \bR$ is a positive definite kernel by construction and we have,
\begin{gather}
  \begin{aligned}
    k_{\Gamma}(i,j) &= \sum_{d=0}^{\min(i,j)-1} C_{\text{hv}}^{i+j-2-2d} C_{\text{d}}^d \frac{(i+j-2-d)!}{(i-1-d)!(j-1-d)!d!}\\
    &= \sum_{d=0}^{\min(i,j)-1} k_d(i, j) =\nonumber
    \sum_{d=0}^{\min(i, j)-1} \delta_{d<i}k_d(i, j)\delta_{d<j}\\
    &=\sum_{d=0}^{\min(i,j)-1} \hat{k}_d(i,j)\nonumber
    = \sum_{d=0}^{\infty} \hat{k}_d(i, j)\nonumber,
  \end{aligned}
\end{gather}
where $\delta_{d<i}=1$ if $d<i$ and zero otherwise. For any finite set of integers, only finitely many terms in the sum above are non-zero and the kernel $k_{S}$ is clearly positive since it is a sum of positive kernels.
\begin{cor}
Let $k_{\Sigma}:\Sigma\times \Sigma\to \bR$ be a continuous positive definite kernel on $\Sigma$, where $\Sigma$ is a separable metric space. Then the associated path kernel $k_{p}:Seq(\Sigma)\times Seq(\Sigma)\to \bR$ is a positive definite kernel for any $C_{\text{d}}\ge 0$ and $C_{\text{hv}}\in \bR$.
\end{cor}

\section{Experiments}\label{sec:experiments}
\begin{figure*}
  \begin{center}
    \includegraphics[width=0.3\textwidth]{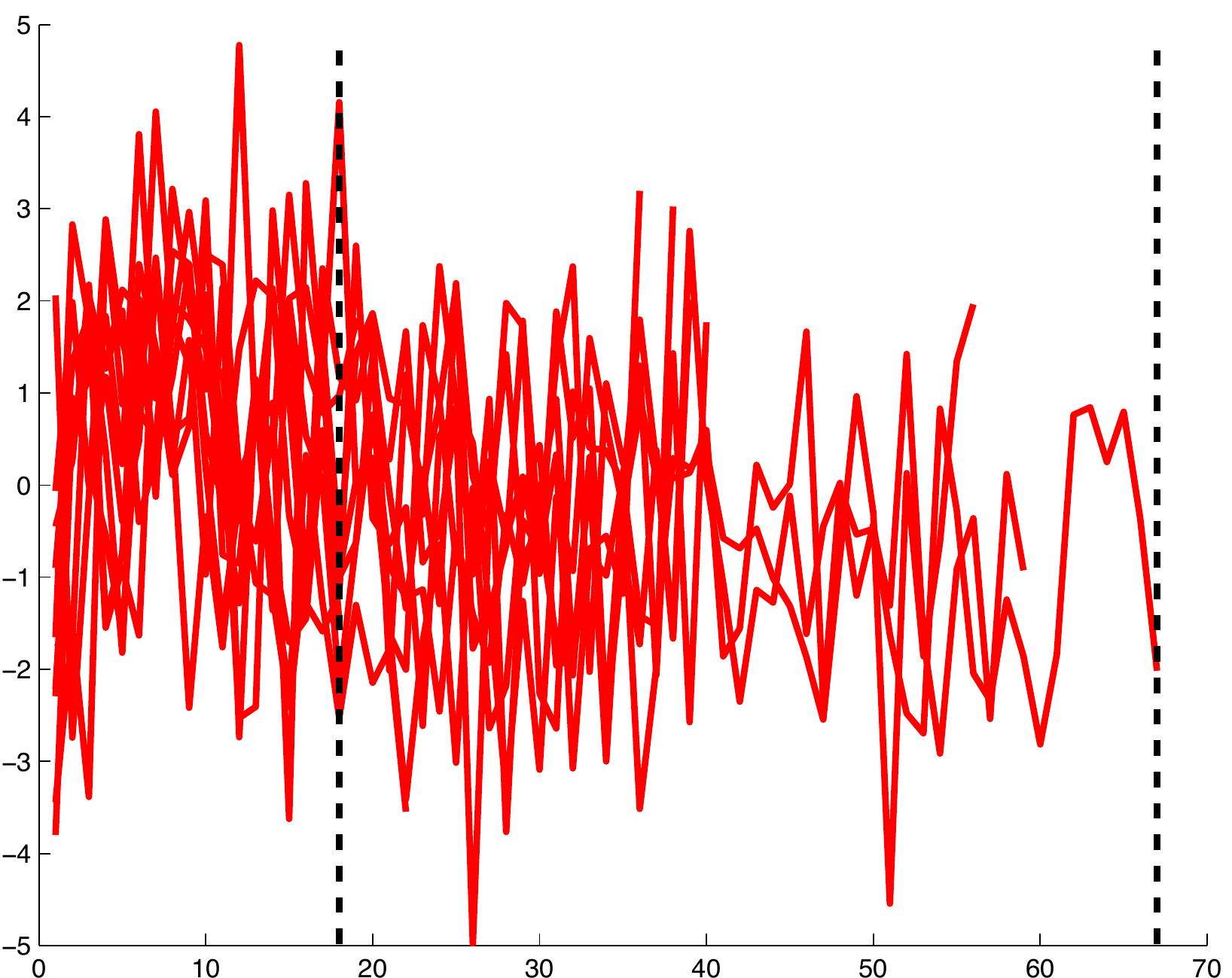}
    \includegraphics[width=0.3\textwidth]{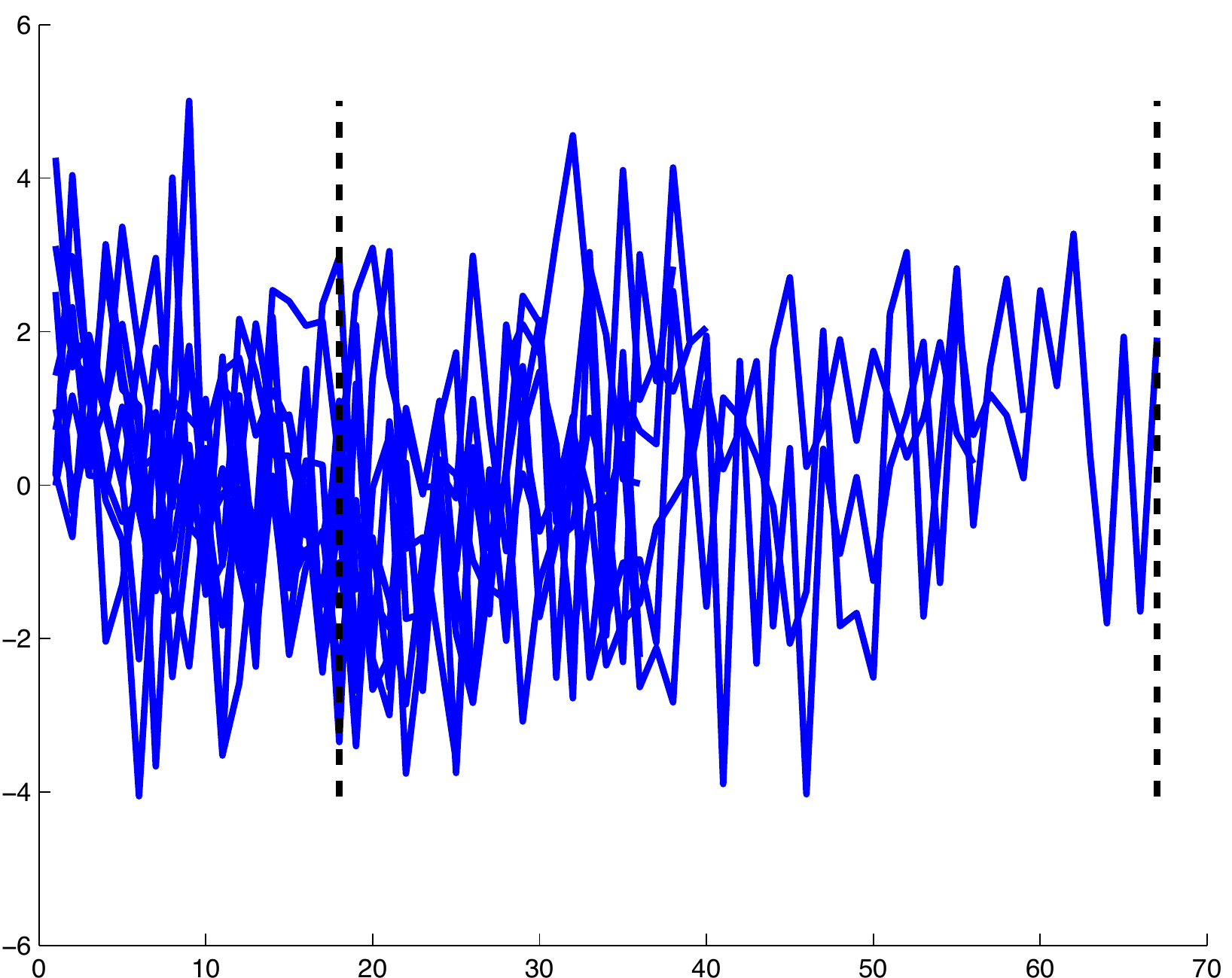}
    \includegraphics[width=0.3\textwidth]{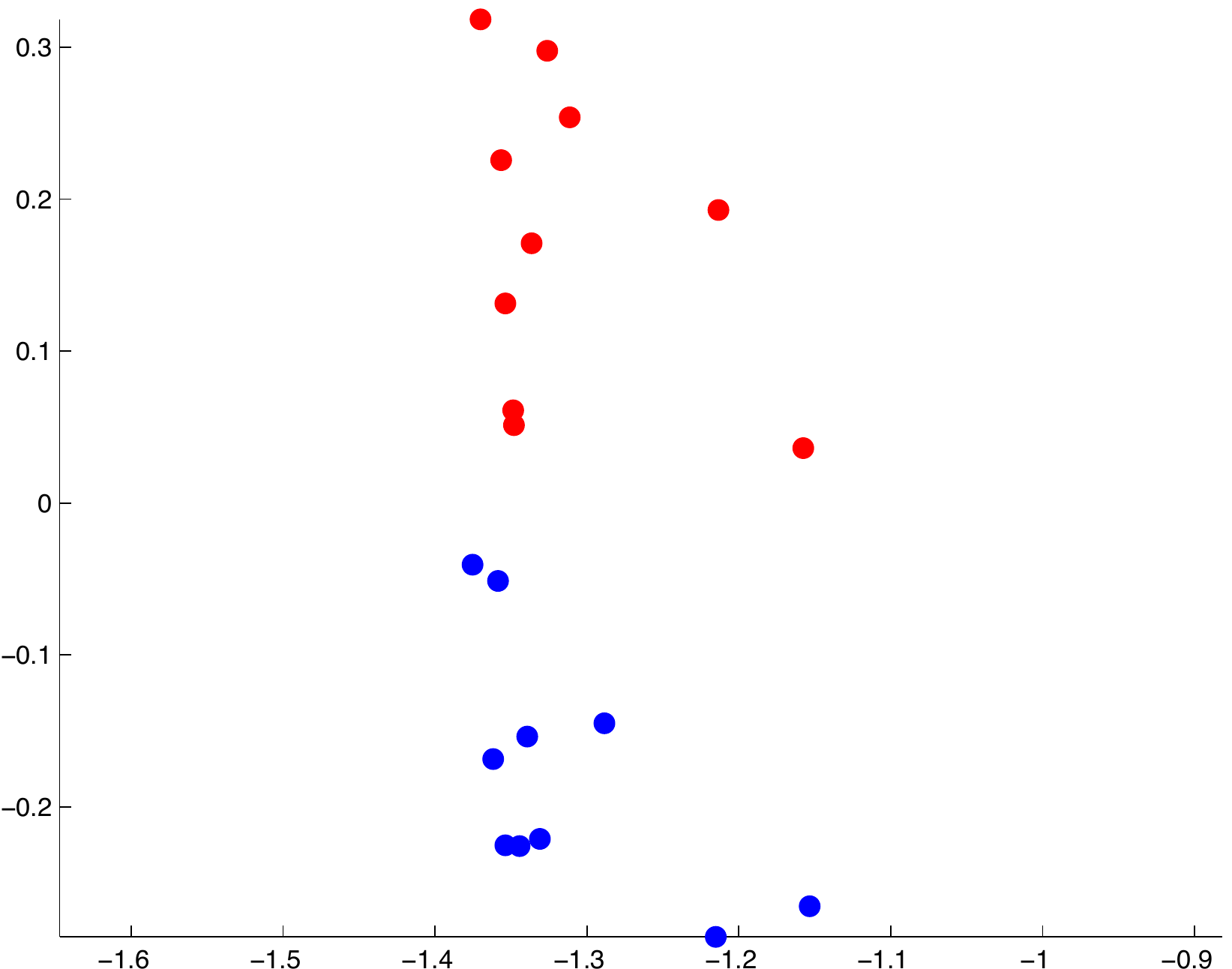}
  \end{center}
  \caption{{\it The above figure depicts the the first toy data-set onto which we have applied the path sequence kernel. The left-most image shows the first class which consists of $10$ noisy sine curves of different length while the middle image shows the other class which consists of noisy cosine curves. The dashed line indicate the length of the longest and the shortest curve in each data-set respectively. The right-most image shows the embedding defined by the first two principal components given by the trained kernel.}}
  \label{fig:gp1}
\end{figure*}

In this section we will experimentally evaluate the performance of the path sequence kernel on a set of real sequential classification data-sets. However, to provide intution for the approach, we will first show how the proposed path sequence kernel represent two sets of toy-data. One of the motivations behind this work is to provide a vectorial embedding of sequences of different length. To evaluate this we, generate $10$ noisy sine and cosine curves of different length as shown in Figure~\ref{fig:gp1}. We now wish to find an embedding that separates the two classes of curves. By formulating the classification task as a regression problem to a $1-C$ encoding and placing a Gaussian Process prior \citep{Rasmussen:2005te} over the mapping, we can learn the kernel parameters through a maximum-likelihood approach. In Figure~\ref{fig:gp1}, the resulting embedding is displayed, clearly showing how the kernel manages to separate the two classes. It is important to note that both the sine and the cosine curves are generated over a full period which means that the first order statistics for the curves are the same. The discriminating information in the data is hence contained in the sequential structure which the Path sequence kernel extracts.

One of the benefits of the proposed decomposable kernel is that, by learning its parameters, we can determine if the discriminating information is contained in the structure or symbol level of the sequences. To evaluate this we generated a second toy data set Figure~\ref{fig:gp2}. The data-set consists of $10$ noisy sine and square waves. Half of the sequences from each class have been altered such that $5$ symbols at random places take the value $4$. We will conduct two experiments on this data. In the first we will learn the parameters of the kernel as to separate the sine from the square waveform, while for the second experiment, we want to separate the sequences containing the randomly positioned symbol $4$ irrespective of waveform. In the first experiment the structure of the symbols are much more important while in the second the only distinguishing aspect is that the sequence contains the symbol $4$. This is reflected by the learned parameters, for the first experiment the coefficent for making diagonal moves, $C_D$, which reflects the importance of the sequences being aligned, is much higher compared to the second experiment where there is little difference between diagonal and horizontal move reflecting that the information is contained at the symbol level of the sequences. In Figure~\ref{fig:gp2} the embedding and the kernel matrices are shown.

\begin{figure*}[t]
  \begin{center}
    \includegraphics[width=0.24\textwidth]{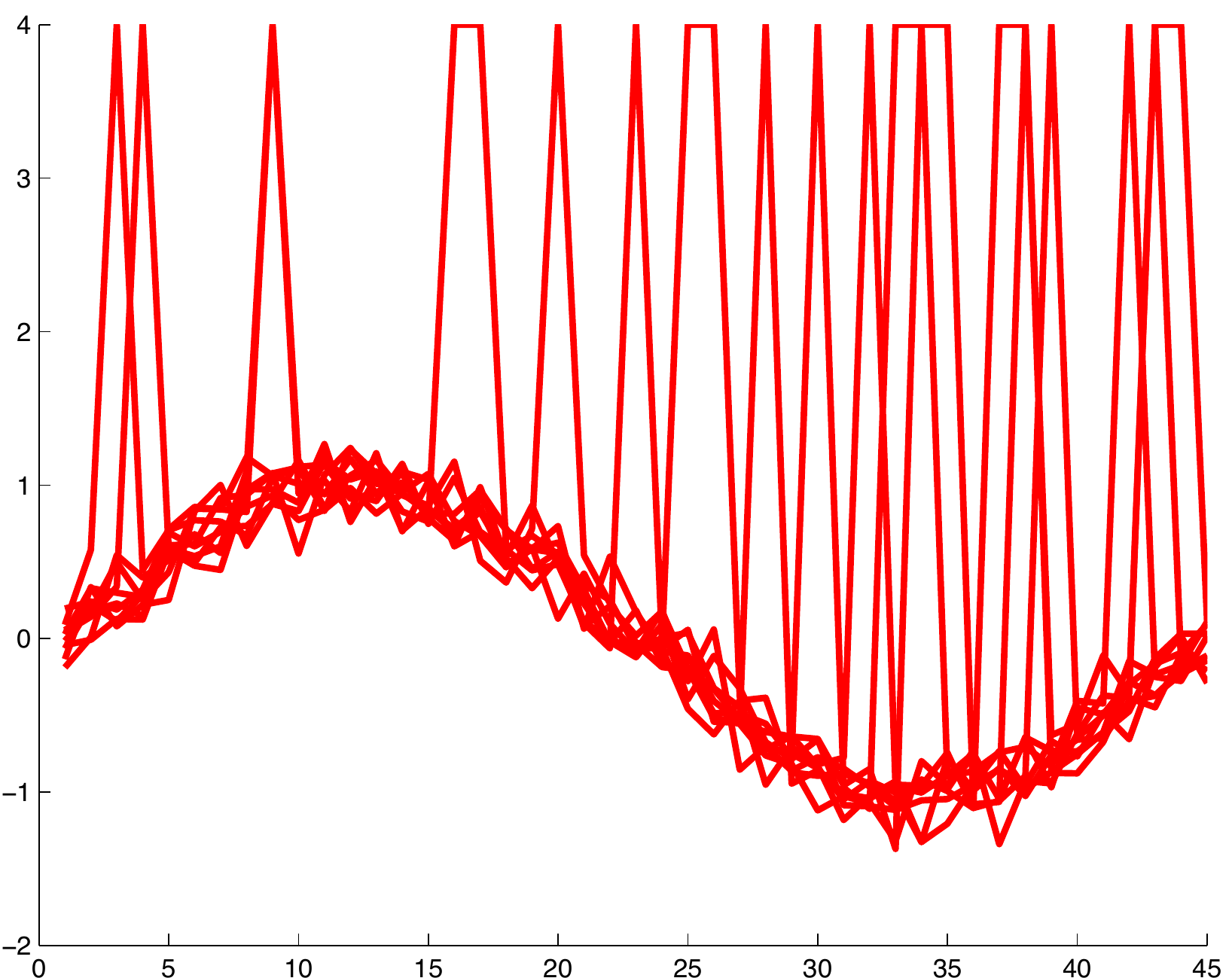}
    \includegraphics[width=0.24\textwidth]{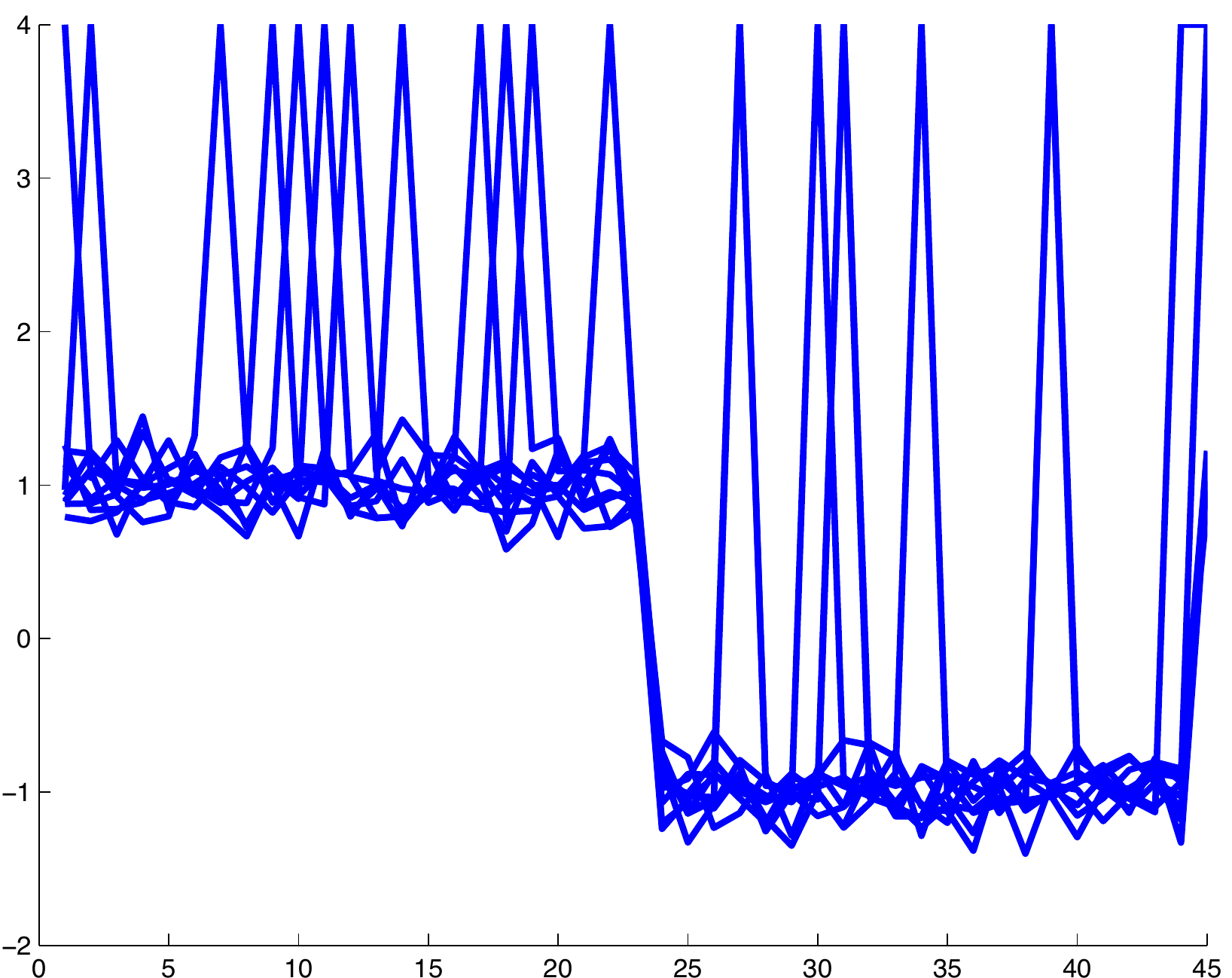}
    \includegraphics[width=0.24\textwidth]{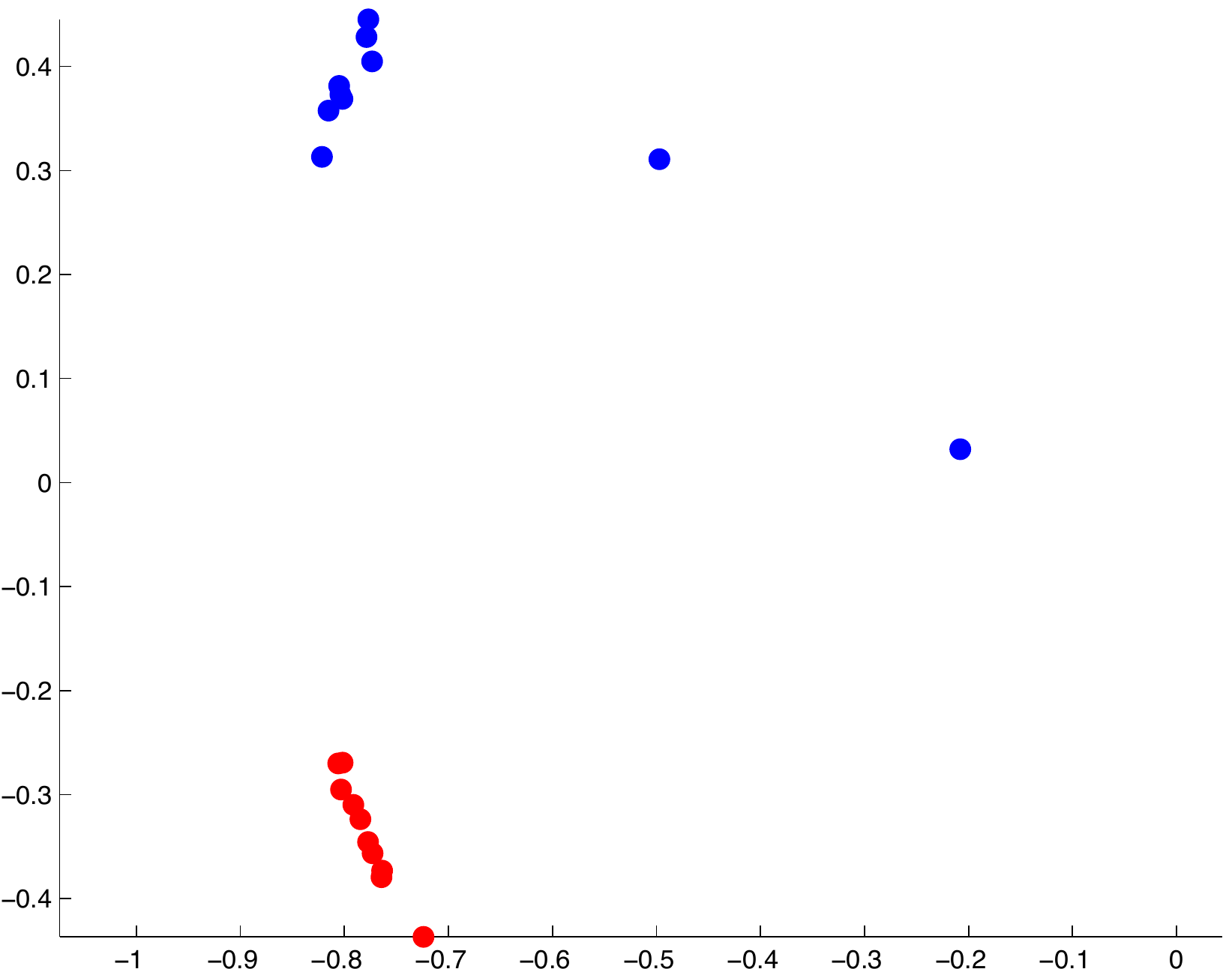}
    \includegraphics[width=0.24\textwidth]{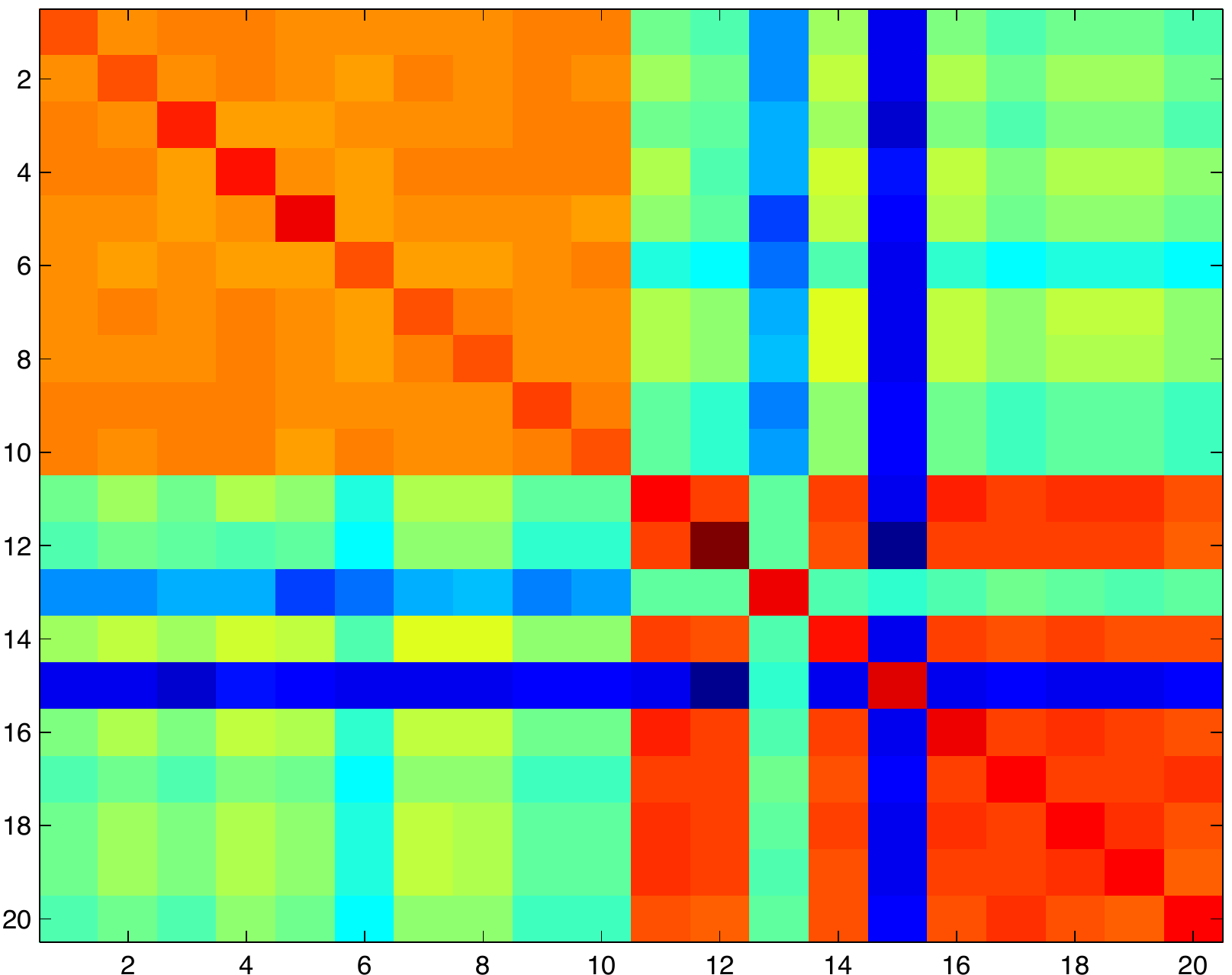}\\
    \includegraphics[width=0.24\textwidth]{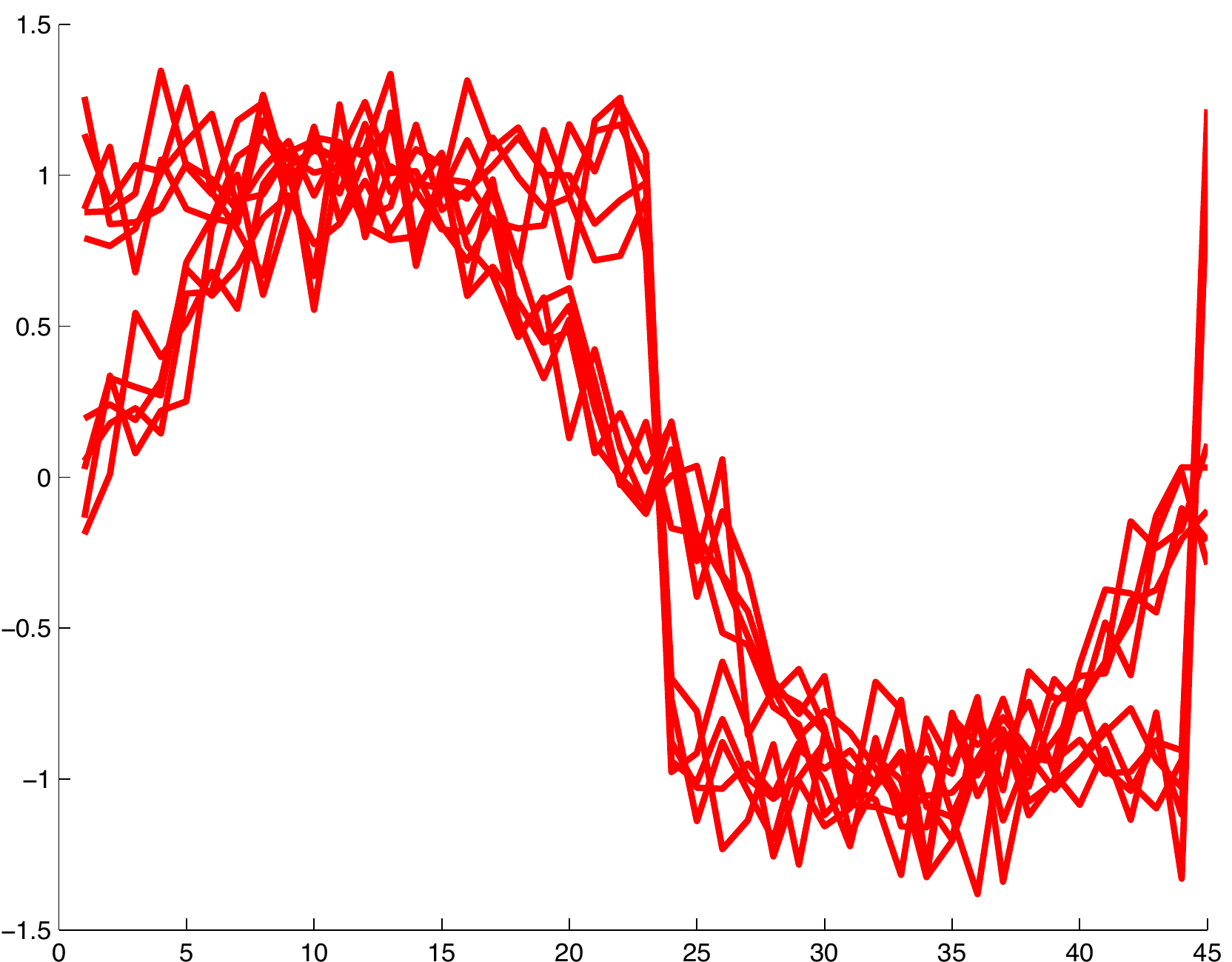}
    \includegraphics[width=0.24\textwidth]{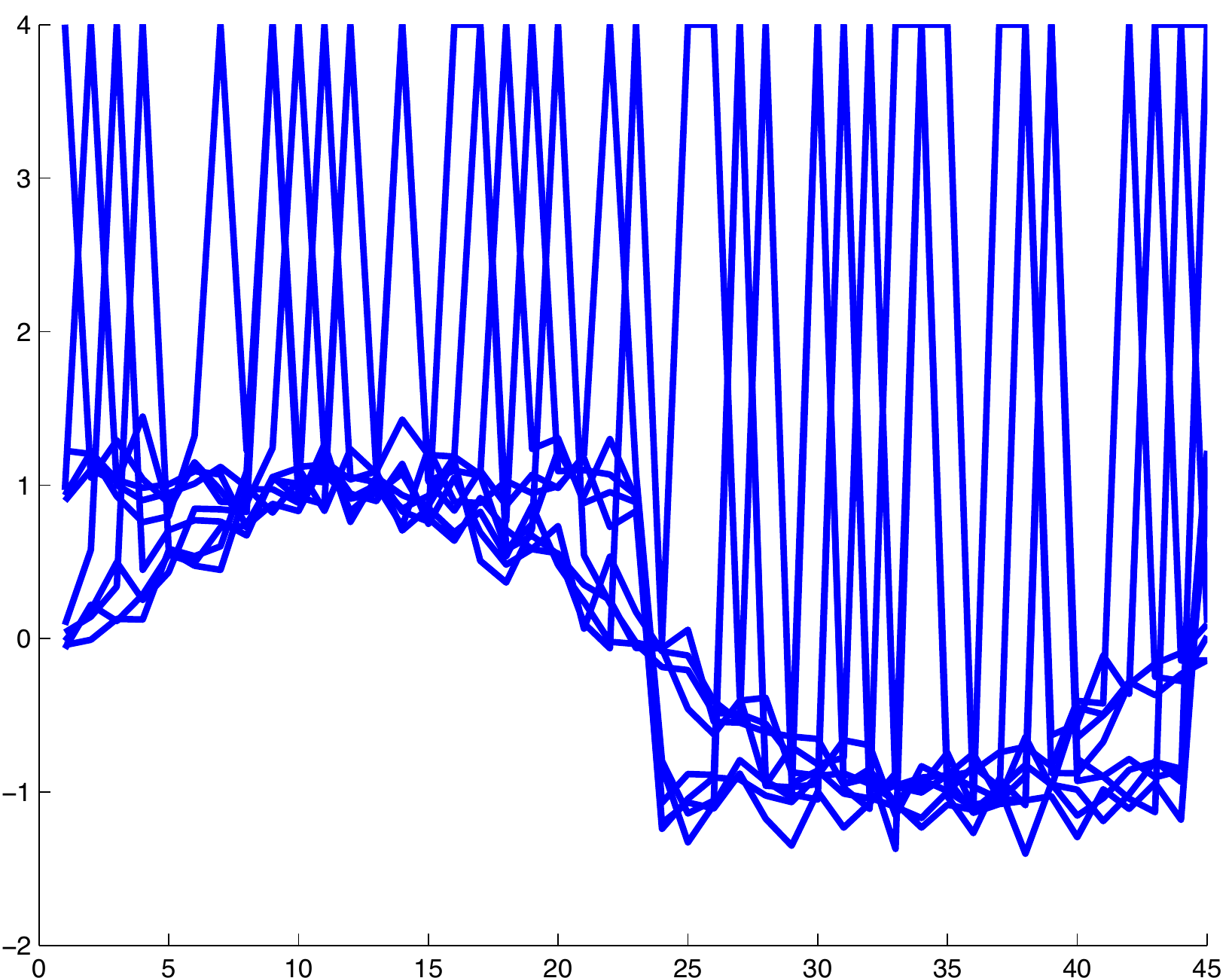}
    \includegraphics[width=0.24\textwidth]{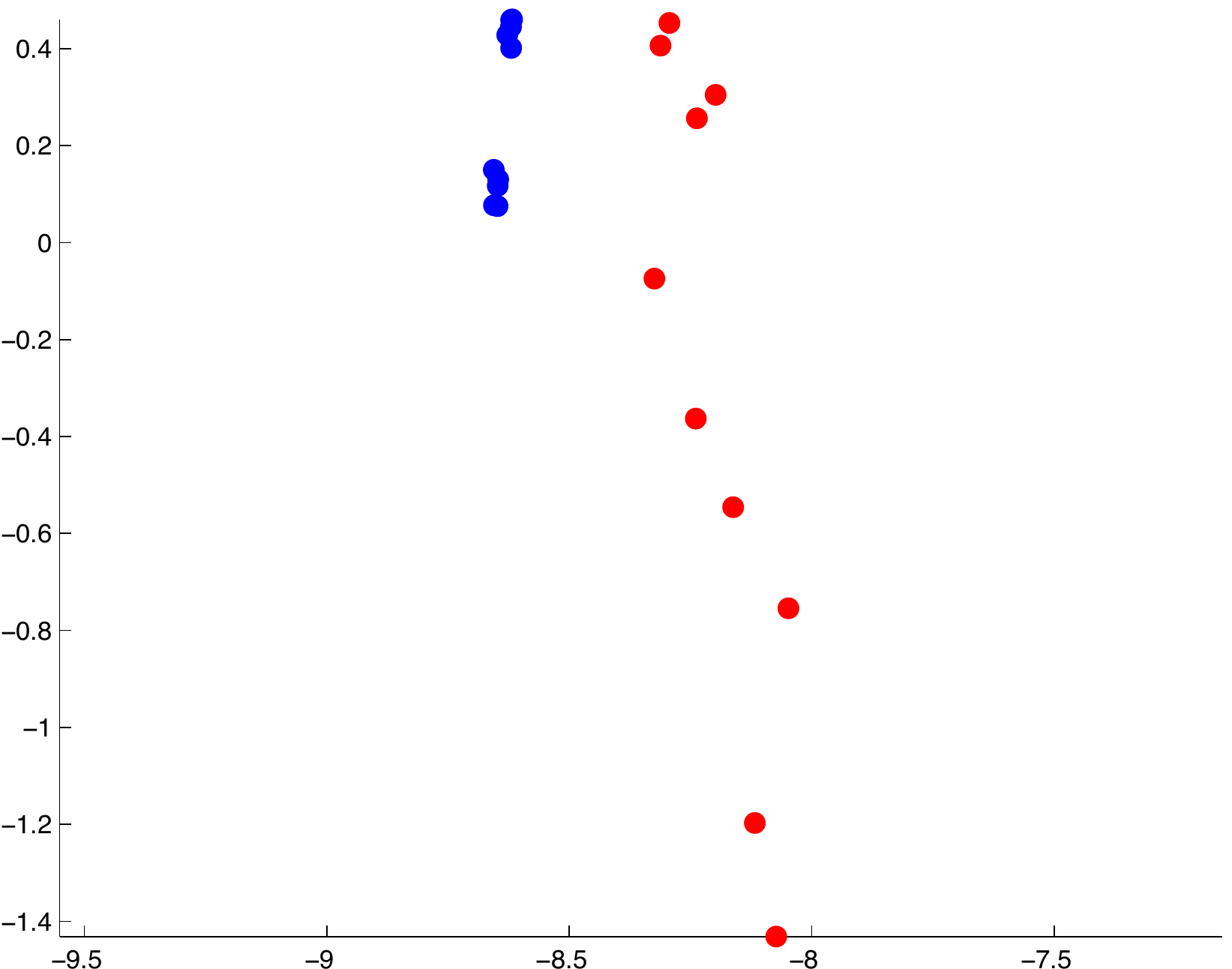}
    \includegraphics[width=0.24\textwidth]{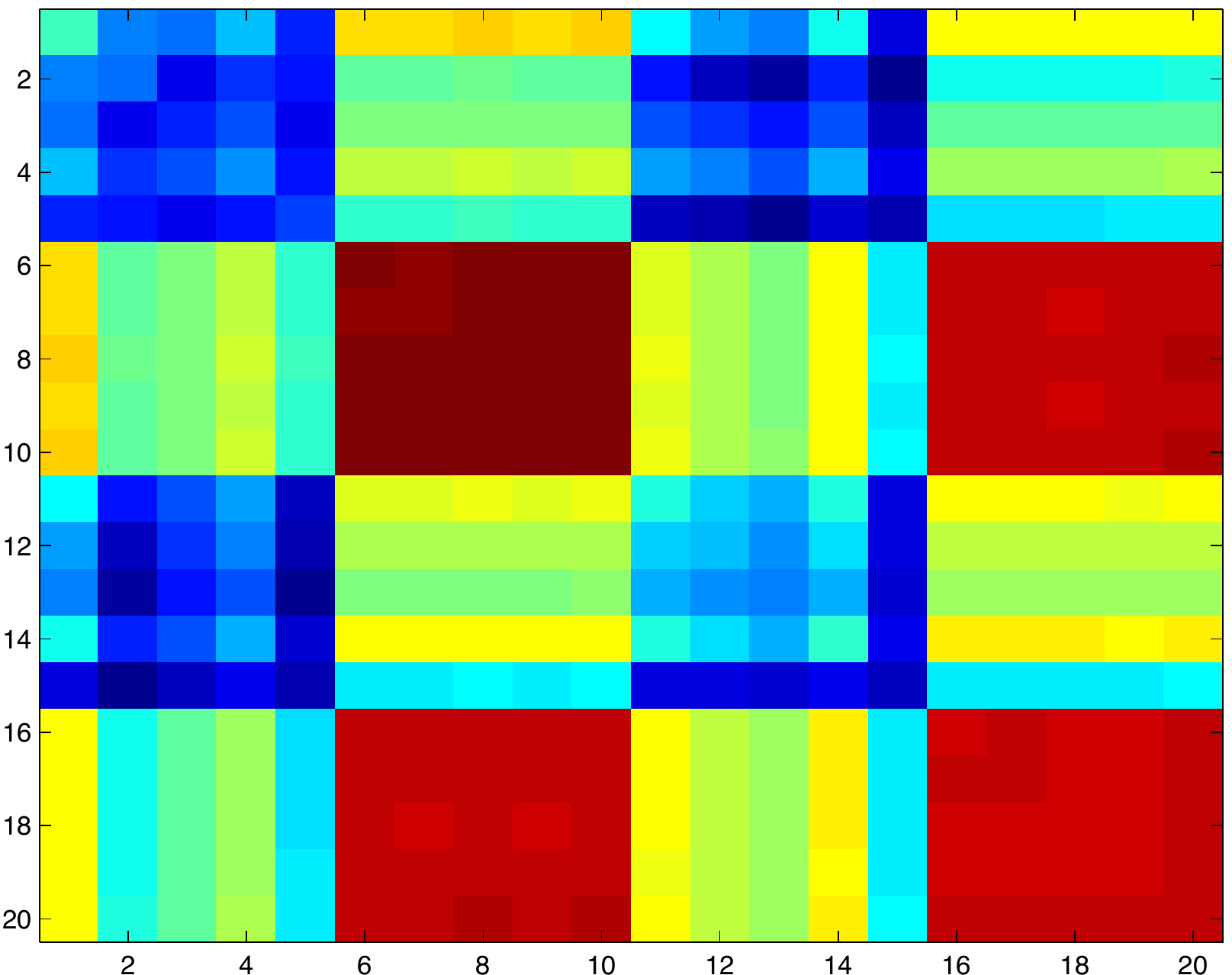}\\
  \end{center}
  \caption{{\it The figure shows the second toy data-set presented in the paper. Each row corresponds to a specific experiment setting. The first two columns show the two input classes while the third shows the embedding defined by the two first principal components and the last column depicts the learned kernel matrix evaluated on the input data.}}
  \label{fig:gp2}
\end{figure*}

We will now proceed to evaluate the performance of the different kernels on a set of well known sequential classification data-sets from the UCI Machine Learning repository \citep{Bache+Lichman:2013} with varying length, dimension and number of classes, see Table~\ref{tab:data}. We compared three different kernels: the exponential, the path sequence kernels and the global alignment kernel\cite{Cuturi:2007db}. In each of the experiments the same exponential kernel is used to represent the symbol space such that the kernels only differ in how they represent the structural component of each sequence. Classification is performed by applying a support vector machine \cite{CC01a} to the space induced by the various kernel. The parameters of the kernels and the classifier are learned using nested-cross validation with $3$ outer and inner folds and $3$ and $20$ repetitions respectively. The outer cross-validation iterates through different divisions of the data into training and test sets. The inner cross-validation uses the training set to perform parameter selection, using the established number of folds and repetitions. The chosen parameters are finally used to test the model on the test set and the results are measured and averaged over the previously described number of folds and repetitions.

To make sure that the discriminative information in the data resides in the structure and not only in the first order statistics of the symbol space we tried to classify the fixed length data by assuming each dimension to be independent. To do so we concatenated each symbol in a sequence and used an exponential kernel and a Euclidean distance as a similarity measure. The results are shown in Table~\ref{tab:results}. As can be seen, the performance for the Libras and the different PEMS data-sets are roughly random indicating that the structure is indeed important.
\begin{table*}[t]
  \begin{center}
    \scalebox{0.71}{\input{includes/data.tex}}
  \end{center}
  \caption{{\it The above table charaterizes the UCI \citep{Bache+Lichman:2013} data-sets on which our experimental evaluation is performed. From top to bottom the data sets where presented in \citep{kadous2002temporal,Dias:2009kt,Cuturi:2010vq,Kudo:1999hq,Williams:2006ce}. The column from left to right show the dimensionality of the symbol space, the range of lengths of the sequences and their median length within brackets, number of classes and the number of sequences. The three different PEMS data-sets are projections of the data onto its principal components such that 100,95 and 90 percent of the variance in the data is retained.}}
  \label{tab:data}
\end{table*}
Comparing the three different kernels, we can see that the path sequence kernel is consistently outperforming the other two kernels. It is interesting to see the significant difference in performance between the exponential and the path sequence kernels. We argue that the difference in performance is due to the stationary characteristics of the exponential kernel where the influence of each symbol match only depends on the difference in position. The path kernel has a more fine-grained characteristic where the actual position of a match incluences the similarity score. Both the path and the global alignment kernel take all possible alignments into account. As we see, the path kernel significantly outperforms the global alignment kernel. This can be explained by the dominating influence of the ``best'' alignment in the global alignment kernel compared to the path kernel which more gracefully accumulates information from all possible alignments into the final kernel value. This behavior also means that there is a strong preference towards sequences of the same length for the global alignment kernel which is something that can explain the big difference in performance compared with the path kernel on the AUSLAN data-set. We believe this shows the value of a non-stationary structured kernel for representing sequences. The path perspective provides an intuitive, rigorous and interpretable formulation for designing such new kernels.

\section{Conclusion}\label{sec:conclusion}
In this paper we have presented an approach to combine kernels in a structured manner such that the resulting measurement represent a Mercer kernel. This leads to kernels that provides a (dis)similarity measure between sequences of different length. In particular we proved that a recently proposed (dis)similarity measure \citep{Baisero:2013tm} falls within this family which allows us to adopt the intuitive notion of alignments which we believe will be provide useful insights for designing new kernels. We showed experimentally how the path sequence kernel significantly outperforms previous state-of-the-art methods. In future work we aim to further establish the path perspective and discuss how new novel kernels which include higher-order paths which takes more than simple moves into account can be designed.

\begin{table*}[t]
  \begin{center}
    \scalebox{0.69}{\input{includes/table2.tex}}
  \end{center}
  \label{tab:results}
  \caption{{\it The above table shows the average classification performance and the standard deviation over the multiple runs. The symbols $k_{ga}$,$k_e$ and $k_p$ refers to the global alignment, exponential sequence and the path sequence kernel respectively.}}
\end{table*}

\bibliographystyle{natbib}
\bibliography{carl}
%\printbibliography

\end{document}

%% file: preamble.tex
\usepackage{nicefrac}

\newcommand{\rd}{{\mathrm d}}

\newcommand{\bC}{\mathbb{C}}

\newcommand{\bR}{\mathbb{R}}

\newcommand{\bN}{\mathbb{N}}

\renewcommand{\phi}{\varphi}

\renewcommand{\ge}{\geqslant}
\renewcommand{\le}{\leqslant}
\newtheorem{thm}{Theorem}[section]
\newtheorem{lem}[thm]{Lemma}        % Lemma environment
\newtheorem{cor}[thm]{Corollary}        % Lemma environment
\tikzset{
    table/.style={
        matrix of nodes,
        row sep=-\pgflinewidth,
        column sep=-\pgflinewidth,
        nodes={
            rectangle,
            draw=black,
            align=center
        },
        minimum height=1.5em,
        text depth=0.5ex,
        text height=2ex,
        nodes in empty cells,
        every even row/.style={
            nodes={fill=gray!20}
        },
        column 1/.style={
            nodes={text width=2em,font=\bfseries}
        },
        row 1/.style={
            nodes={
                fill=black,
                text=white,
                font=\bfseries
            }
        }
    }
}

\newcommand{\norm}[1]{\lVert #1 \rVert}
\newcommand{\abs}[1]{\left| #1 \right|}

\newcommand{\subsumk}{k_\text{SS}}
\newcommand{\subsumkl}[1]{\subsumk^{(l)}}

%% file: includes/pathkernel.tex
\begin{tikzpicture}[scale=1.7]

\tikzset{
  ashadow/.style={opacity=.25, shadow xshift=0.07, shadow yshift=-0.07},
}

  \begin{scope}
    \foreach \y in {1,2,3}{
      \node (0,\y) (lineystart_\y) {};
      \node (6,\y) (lineyend_\y) {};
      \draw[name=liney] (0,\y) -- (6,\y) {};
    }
    \foreach \x in {1,2,3,4,5}{
      \node (\x,0) (linexstart_\x) {};
      \node (\x,4) (linexend_\x) {};
      \draw (\x,0) -- (\x,4) {};
    }

    \foreach \x in {1,2,3,4,5,6}{
      \node[anchor=center] at (\x-0.5,4.35) (s_\x) {\huge $s_{\x}$};
    }     
    \foreach \y in {1,2,3,4}{
      \node[anchor=center] at (-0.5,-\y+4.5) (t_\y){\huge $t_{\y}$};
    }

    % draw paths
    \draw[->,line width=5pt,draw=red,opacity=0.5] (0.5,3.5)--(0.5,2.5)--(2.5,2.5)--(3.5,1.5)--(5.5,1.5)--(5.5,0.5) {};
    \draw[->,line width=5pt,draw=blue,opacity=0.5] (0.5,3.5)--(1.5,3.5)--(2.5,3.5)--(3.5,3.5)--(3.5,2.5)--(3.5,1.5)--(4.5,0.5)--(5.5,0.5){};
    \draw[->,line width=5pt,draw=green,opacity=0.5] (0.5,3.5) -- (0.5,0.5) -- (5.5,0.5) {};

    \begin{scope}[on background layer,scale=0.55]
      \node[drop shadow={shadow xshift=.17cm, shadow yshift=-0.17cm},rectangle,draw,ultra thick,fit=(s_1)(s_6)(t_4),fill=white,draw=black,rounded corners=0.2cm] {};
    \end{scope}
    
  \end{scope}

  \begin{scope}[xshift=7cm,yshift=0.3cm]

    \node[anchor=center] at (0,3.5) (g1_1) {\huge $\gamma_1$ };
    % Blue
    \node[anchor=center] at (1.2,3.9) (gamma1_top_1) {\huge $s_1$};
    \node[anchor=center] at (2.2,3.9) (gamma1_top_2) {\huge $s_2$};
    \node[anchor=center] at (3.2,3.9) (gamma1_top_3) {\huge $s_3$};
    \node[anchor=center] at (4.2,3.9) (gamma1_top_4) {\huge $s_4$};
    \node[anchor=center] at (5.2,3.9) (gamma1_top_5) {\huge $s_4$};
    \node[anchor=center] at (6.2,3.9) (gamma1_top_6) {\huge $s_4$};
    \node[anchor=center] at (7.2,3.9) (gamma1_top_7) {\huge $s_5$};
    \node[anchor=center] at (8.2,3.9) (gamma1_top_8) {\huge $s_6$};

    \node[anchor=center] at (1.2,3.1) (gamma1_bottom_1) {\huge $t_1$};
    \node[anchor=center] at (2.2,3.1) (gamma1_bottom_2) {\huge $t_1$};
    \node[anchor=center] at (3.2,3.1) (gamma1_bottom_3) {\huge $t_1$};
    \node[anchor=center] at (4.2,3.1) (gamma1_bottom_4) {\huge $t_1$};
    \node[anchor=center] at (5.2,3.1) (gamma1_bottom_5) {\huge $t_2$};
    \node[anchor=center] at (6.2,3.1) (gamma1_bottom_6) {\huge $t_3$};
    \node[anchor=center] at (7.2,3.1) (gamma1_bottom_7) {\huge $t_4$};
    \node[anchor=center] at (8.2,3.1) (gamma1_bottom_8) {\huge $t_4$};

    \foreach \x in {1,2,3,4,5,6,7,8}{
      \draw[<->,ultra thick] (gamma1_top_\x) -- (gamma1_bottom_\x){};
    }

    \node[anchor=center] at (0,2) (g2_1){\huge $\gamma_2$ };
    % Red
    \node[anchor=center] at (1.7,2.4) (gamma2_top_1) {\huge $s_1$};
    \node[anchor=center] at (2.7,2.4) (gamma2_top_2) {\huge $s_1$};
    \node[anchor=center] at (3.7,2.4) (gamma2_top_3) {\huge $s_2$};
    \node[anchor=center] at (4.7,2.4) (gamma2_top_4) {\huge $s_3$};
    \node[anchor=center] at (5.7,2.4) (gamma2_top_5) {\huge $s_4$};
    \node[anchor=center] at (6.7,2.4) (gamma2_top_6) {\huge $s_5$};
    \node[anchor=center] at (7.7,2.4) (gamma2_top_7) {\huge $s_6$};
    \node[anchor=center] at (8.7,2.4) (gamma2_top_8) {\huge $s_6$};

    \node[anchor=center] at (1.7,1.6) (gamma2_bottom_1) {\huge $t_1$};
    \node[anchor=center] at (2.7,1.6) (gamma2_bottom_2) {\huge $t_2$};
    \node[anchor=center] at (3.7,1.6) (gamma2_bottom_3) {\huge $t_2$};
    \node[anchor=center] at (4.7,1.6) (gamma2_bottom_4) {\huge $t_2$};
    \node[anchor=center] at (5.7,1.6) (gamma2_bottom_5) {\huge $t_3$};
    \node[anchor=center] at (6.7,1.6) (gamma2_bottom_6) {\huge $t_3$};
    \node[anchor=center] at (7.7,1.6) (gamma2_bottom_7) {\huge $t_3$};
    \node[anchor=center] at (8.7,1.6) (gamma2_bottom_8) {\huge $t_4$};

    \foreach \x in {1,2,3,4,5,6,7,8}{
      \draw[<->,ultra thick] (gamma2_top_\x) -- (gamma2_bottom_\x){};
    }

    \node at (0,0.5) (g3_1){\huge $\gamma_3$ };
    % Green
    \node[anchor=center] at (1,0.9) (gamma3_top_1) {\huge $s_1$};
    \node[anchor=center] at (2,0.9) (gamma3_top_2) {\huge $s_1$};
    \node[anchor=center] at (3,0.9) (gamma3_top_3) {\huge $s_1$};
    \node[anchor=center] at (4,0.9) (gamma3_top_4) {\huge $s_1$};
    \node[anchor=center] at (5,0.9) (gamma3_top_5) {\huge $s_2$};
    \node[anchor=center] at (6,0.9) (gamma3_top_6) {\huge $s_3$};
    \node[anchor=center] at (7,0.9) (gamma3_top_7) {\huge $s_4$};
    \node[anchor=center] at (8,0.9) (gamma3_top_8) {\huge $s_5$};
    \node[anchor=center] at (9,0.9) (gamma3_top_9) {\huge $s_6$};

    \node[anchor=center] at (1,0.1) (gamma3_bottom_1) {\huge $t_1$};
    \node[anchor=center] at (2,0.1) (gamma3_bottom_2) {\huge $t_2$};
    \node[anchor=center] at (3,0.1) (gamma3_bottom_3) {\huge $t_3$};
    \node[anchor=center] at (4,0.1) (gamma3_bottom_4) {\huge $t_4$};
    \node[anchor=center] at (5,0.1) (gamma3_bottom_5) {\huge $t_4$};
    \node[anchor=center] at (6,0.1) (gamma3_bottom_6) {\huge $t_4$};
    \node[anchor=center] at (7,0.1) (gamma3_bottom_7) {\huge $t_4$};
    \node[anchor=center] at (8,0.1) (gamma3_bottom_8) {\huge $t_4$};
    \node[anchor=center] at (9,0.1) (gamma3_bottom_9) {\huge $t_4$};
    \foreach \x in {1,2,3,4,5,6,7,8,9}{
      \draw[<->,ultra thick] (gamma3_top_\x) -- (gamma3_bottom_\x){};
    }

    %\begin{pgfonlayer}{background}
   \begin{scope}[on background layer,scale=0.6]
    \node[drop shadow={shadow xshift=.17cm, shadow yshift=-0.17cm}, rectangle,draw,fit={(gamma1_bottom_1)(gamma1_top_8)},fill=blue!30,draw=black,rounded corners=0.2cm] (g1_box){};
    \node[drop shadow={shadow xshift=.17cm, shadow yshift=-0.17cm}, rectangle,draw,fit=(gamma2_bottom_1)(gamma2_top_8),fill=red!30,draw=black,rounded corners=0.2cm] (g2_box){};
    \node[drop shadow={shadow xshift=.17cm, shadow yshift=-0.17cm}, rectangle,draw,fit=(gamma3_bottom_1)(gamma3_top_9),fill=green!30,draw=black,rounded corners=0.2cm] (g3_box){};
    
   \end{scope}
  %\end{pgfonlayer}

  \end{scope}

  \draw[->,very thick] (g1_1) to[out=180,in=0] (3.5,2.6){};
  \draw[->,very thick] (g1_1) to[out=0,in=180] (g1_box.west){};
  \draw[->,very thick] (g2_1) to[out=180,in=90] (5.2,1.5){};
  \draw[->,very thick] (g2_1) to[out=0,in=180] (g2_box.west){};
  \draw[->,very thick] (g3_1) to[out=180,in=90] (2.5,0.5){};
  \draw[->,very thick] (g3_1) to[out=0,in=180] (g3_box.west){};

\end{tikzpicture}

%% file: includes/data.tex
\begin{tikzpicture}
  \matrix (first) [table,text width=11em]
{
  \node[text width=8em]{};& dim & length & \#classes & \#N\\
  \node[text width=8em]{AUSLAN}; & 22 & 45-136 (55) & 95 & 2565 \\
  \node[text width=8em]{Libras}; & 2 & 45 & 15 & 945 \\
  \node[text width=8em]{PEMS100}; & 963 & 144 & 7 & 440 \\
  \node[text width=8em]{PEMS95}; & 335 & 144 & 7 & 440 \\
  \node[text width=8em]{PEMS90}; & 171 & 144 & 7 & 440 \\
  \node[text width=8em]{Vowels}; & 12 & 7-29 (15) & 9 & 640 \\
  \node[text width=8em]{Characters}; & 3 & 60-182 (122) & 20 & 2858\\
};
\end{tikzpicture}

%% file: includes/table2.tex
\begin{tikzpicture}
  \matrix (first) [table,text width=7em]
{
& AUSLAN   & Libras & PEMS100 & PEMS95 & PEMS90 & Vowels & Characters\\
RBF   & - & $6.67\pm0.41\%$ & $14.13\pm1.45\%$ & $14.74\pm1.42\%$ & $14.11\pm1.29\%$ & - & -\\
$k_{ga}$   & {$68.08\pm8.40\%$} & $84.18\pm3.99\%$ & $39.16\pm17.10\%$ & $39.27\pm19.45\%$ & $36.44\pm19.44\%$ & $96.51\pm1.67\%$ & \node[fill=green]{$98.40\pm0.11\%$};\\
$k_{e}$   & $67.06\pm1.99\%$ & $43.55\pm3.83\%$ & $15.65\pm6.17\%$ & $12.73\pm0.9\%$ & $12.58\pm1.10\%$ & $96.87\pm0.78\%$ & -\\
$k_{p}$   & \node[fill=green]{$92.40\pm0.63\%$}; & \node[fill=green]{$87.37\pm2.86\%$}; & \node[fill=green]{$61.32\pm15.17\%$}; & \node[fill=green]{$60.26\pm16.19\%$}; & \node[fill=green]{$54.05\pm17.17\%$}; & \node[fill=green]{$97.97\pm0.55\%$}; & $98.17\pm0.70\%$\\
};
\end{tikzpicture}